\documentclass[preprint]{article}

\usepackage{aistats2026}
\usepackage{xcolor}
\usepackage[export]{adjustbox}
\usepackage{bm}
\usepackage{graphicx}
\usepackage{enumitem}
\usepackage{amsfonts}
\usepackage{hyperref}
\usepackage{amsmath}
\usepackage{amsthm}
\usepackage{multicol}
\newtheorem{definition}{Definition}

\newtheorem{proposition}{Proposition}
\newtheorem{theorem}{Theorem}
\newtheorem{lemma}{Lemma}
\newtheorem{corollary}{Corollary}
\theoremstyle{definition}

\usepackage[round]{natbib}

\newif\ifshowappendix
\showappendixfalse   

\ifshowappendix

\else
\fi

\begin{document}

%

%

\newcommand{\C}{\mathcal{C}}
\newcommand{\Ni}{N_{\textrm{init}}}
\newcommand{\Ch}{\hat{\C}}
\newcommand{\Cm}{\mathbf{C}}
\newcommand{\sH}{\mathcal{H}}
\newcommand{\vu}{\mathbf{u}}
\newcommand{\x}{\mathbf{x}}
\newcommand{\vw}{\mathbf{w}}
\newcommand{\G}{\mathbf{G}}
\newcommand{\U}{\mathbf{U}}
\newcommand{\V}{\mathbf{V}}

\newcommand{\Hcal}{\mathcal H}
\newcommand{\Acal}{\mathcal A}
\newcommand{\ip}[2]{\langle #1,#2\rangle}
\newcommand{\Proj}{P}
\newcommand{\E}{\mathbb E}

\newcommand{\norm}[1]{\left\|#1\right\|}
\newcommand{\op}{\mathrm{op}}
\newcommand{\HS}{\mathcal{S}_2} 
\newcommand{\tr}{\mathrm{trace}}
\newcommand{\Kcal}{\mathcal K}
\newcommand{\dom}{\mathcal D}
\newcommand{\dx}{\mathrm{d}x}

\newcommand{\todo}[1]{{\color{red} TODO: #1}}
\newcommand{\nathan}[1]{{\color{orange} Nathan says: #1}}
\newcommand{\poorbita}[1]{{\color{blue} Poorbita says: #1}}

\twocolumn[

\aistatstitle{Active Subspaces in Infinite Dimension}

\aistatsauthor{ Poorbita Kundu \And Nathan Wycoff}

\aistatsaddress{ 
Public Health Sciences Division \\
Fred Hutchinson Cancer Center
\And  
Department of Mathematics and Statistics \\ University of Massachusetts Amherst} ]

\begin{abstract}
Active subspace analysis uses the leading eigenspace of the gradient's second moment to conduct supervised dimension reduction.
In this article, we extend this methodology to real-valued functionals on Hilbert space.
We define an operator which coincides with the active subspace matrix when applied to a Euclidean space.
We show that many of the desirable properties of Active Subspace analysis extend directly to the infinite dimensional setting.
We also propose a Monte Carlo procedure and discuss its convergence properties.
Finally, we deploy this methodology to create visualizations and improve modeling and optimization on complex test problems.
\end{abstract}

\section{Introduction}

The increasing availability of computational resources has allowed for significantly more sophisticated models in science and engineering.
But this increase in complexity is not always accompanied by a commensurate increase in understanding.
Often, these models are parameterized by a set of variables governing the behavior of the simulation.
When there are many such parameters, a \textit{sensitivity analysis} \citep{razavi2021future} can be key to unlocking knowledge of which are most important in terms of driving the system. 
A linear sensitivity analysis in particular looks for important linear combinations of variables; this can advance understanding of even complex systems where all variables play an important role.

\textit{Gradient-based sensitivity analysis} \citep{peter2010numerical} uses the derivative of a target system's behavior with respect to its parameters to determine importance.
Active subspace analysis \citep{constantine2015active} is the most prominent linear gradient-based method, which given a differentiable function $f: \mathcal{X}\subseteq\mathbb{R}^D \to \mathbb{R}$,
proceeds by defining an \textit{active subspace matrix}:
\begin{equation}
    \Cm = \mathbb{E}_{\rho}
    \left[
    \nabla f(X) \nabla f(X)^\top
    \right] \in\mathbb{R}^{D\times D}\,,
\end{equation}
where $X$ is a $\mathcal X$-valued random variable with probability law $\rho$.
In conducting an eigenanalysis on $\Cm$, we have the opportunity to discover some subspace of $\mathbb{R}^D$ which captures most or even all of $f$'s variation.
For many practical functions, the gradient is restricted primarily to some lower dimensional subspace \citep{constantine2016many}.
The strongest such circumstance is that of a \textit{ridge function}, which has the form $f(x) = \tilde f(\mathbf{A}x)$ for some $\mathbf{A}\in\mathbb{R}^{R\times D}$ with $R<D$.

\begin{figure*}[h]
    \centering
    \newcommand\figscaleone{0.25}

    
    \begin{tabular}[c]{|c|c|c|c|}
    \hline
    &\texttt{poisson} & \texttt{laminar} & \texttt{kiri} \\
    \hline
    $u$ 
    &
    \includegraphics[width=\figscaleone\linewidth,valign=m]{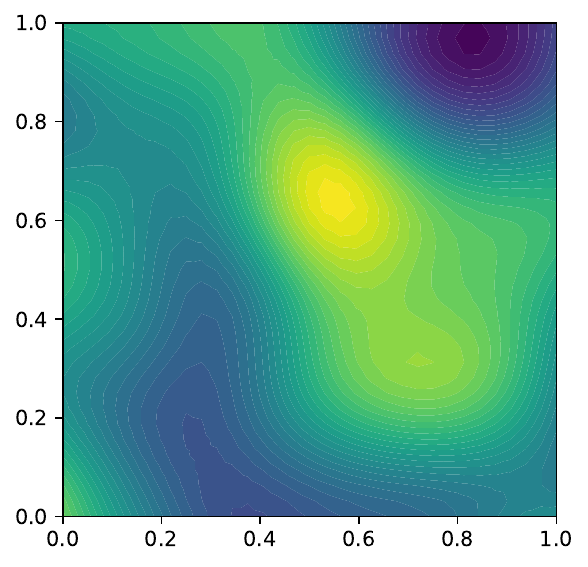}
    &
    \includegraphics[width=\figscaleone\linewidth,valign=m]{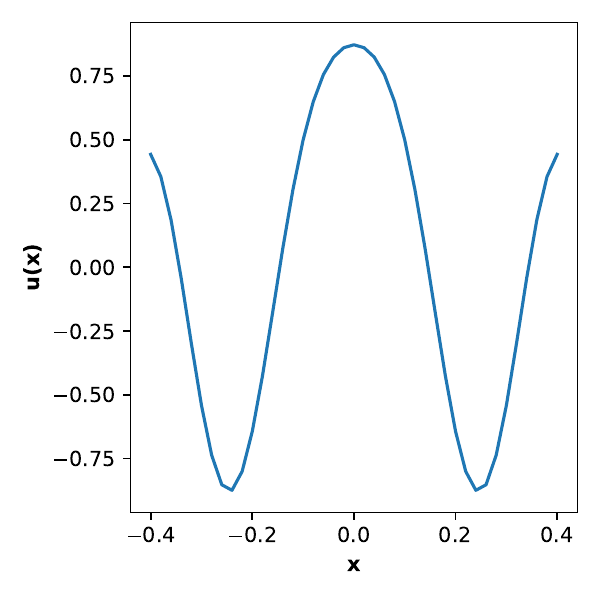}
    &
    \includegraphics[width=\figscaleone\linewidth,trim={4.5em 4.5em 4.5em 4.5em},clip,valign=m]{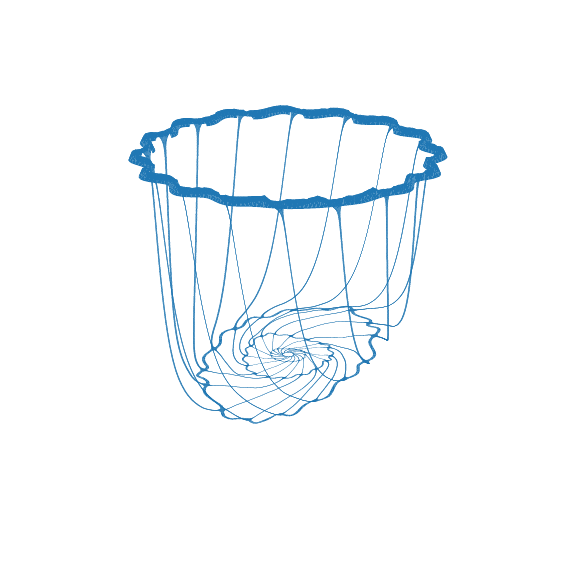} \\
    \hline
    $\nabla f(u)$ & 
    \includegraphics[width=\figscaleone\linewidth,valign=m]{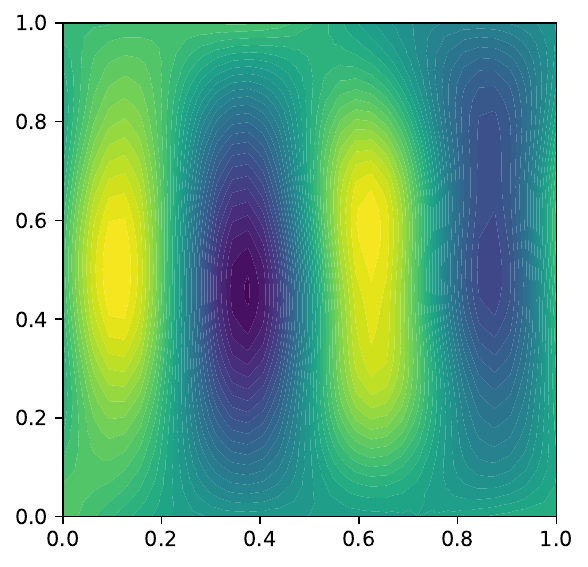}
    &
    \includegraphics[width=\figscaleone\linewidth,valign=m]{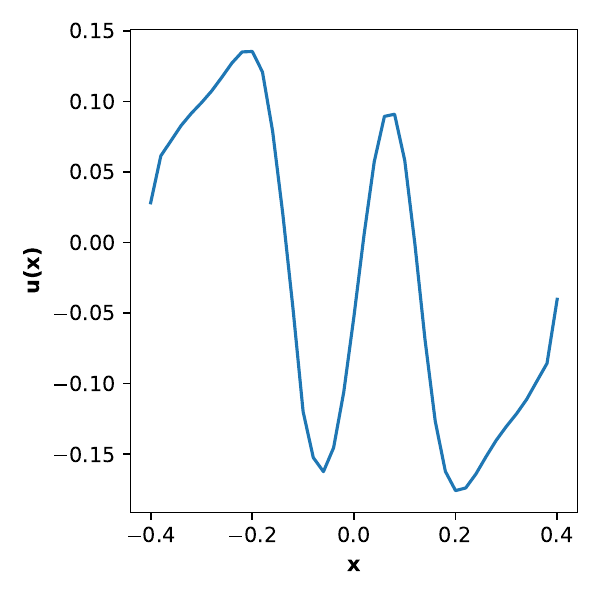}
    &
    \includegraphics[width=\figscaleone\linewidth,trim={6em 6em 6em 6em},clip,valign=m]{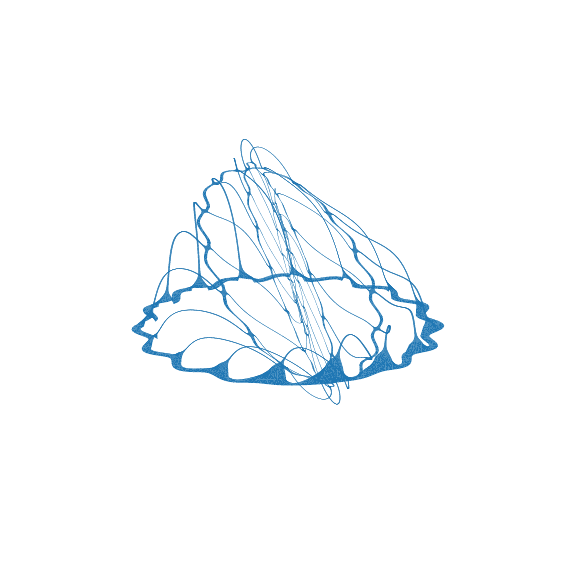} \\
    \hline
    \end{tabular}
    
    \caption{
    Random input functions (top row) and corresponding gradient functions (bottom row).
    }
    \label{fig:illus}
\end{figure*}

The active subspace method has found many applications across science and engineering \citep{kim2024adaptive,grey2018active,khatamsaz2021adaptive,ji2019quantifying,constantine2015discovering,bittner2021temporal}.
In part, this is due to the  proliferation of computing environments which allow for automatic calculation of gradients with respect to model parameters, including deep learning frameworks (e.g. Tensorflow \citep{tensorflow2015-whitepaper}, Pytorch \citep{paszke2019pytorch}, JAX \citep{jax2018github}) and scientific computing packages (e.g. FEniCS \citep{alnaes2015fenics}, custEM \citep{rochlitz2019custem}). 
These allows for the computation of gradients with respect to parameters in arbitrarily high dimension, and even infinite dimensional spaces in the case of e.g. FEniCS.

The computational capabilities of software systems therefore appear to have run ahead of the statistical methodology and mathematical understanding of the active subspace method.
The purpose of this article is to rectify this.
In particular, our primary contributions are as follows:
\begin{enumerate}
    \item We define an extended Active Subspace analysis for real-valued functionals on Hilbert spaces.
    \item We show that many celebrated properties of the active subspace extend to the functional context. 
    \item We propose a computable Monte-Carlo estimator and establish its convergence.
    \item Sophisticated case studies show how to use this active subspace to improve practical analyses.
    \item We release an open-source code-base building on FEniCS that implements this methodology.
\end{enumerate}


\section{Motivation and Background}

Since we never represent functions exactly on a computer, but must instead discretize them at some point, a natural objection to this article's project is that we can always first discretize the function and then apply the active subspace method to the discretization.
We thus provide with a discussion of this idea. 
Next, we provide an overview of the class of test functions, optimization of PDE parameters, which have motivated our methodology.
But first, we overview related work.

\subsection{Sensitivity Analysis in Infinite Dimension}
Though active subspaces themselves have not previously been applied to infinite dimensional problems to the best of our knowledge, we now review existing work on sensitivity analysis more broadly in the infinite dimensional setting.
The procedure in this article will be closely related to functional PCA \citep{ramsay2005functional}, by analogy to the close connection between PCA and the active subspace method in the finite dimensional setting.
In the context of finance, sensitivity analysis is commonly used under the name ``the Greeks"; \citet{benth2021sensitivity} extended these to the infinite dimensional setting.
Sobol indices, an alternative to gradient-based sensitivity, have been studied in infinite dimension \citep[e.g.][]{iooss2009global,gamboa2014sensitivity}.
These are most useful in the observational context, when gradient information is not available.
We will in this article rather be concerned with the analysis of a computational model for which gradients are available.
\citet{wu2023large} proposed an interesting technique for optimal Bayesian experimental design over PDE parameters and \citet{beskos2017geometric} propose using gradients to accelerate Bayesian computation.

\subsection{The Noncommutativity of Discretization}
\citet{diez2024design} review sensitivity analysis with respect to functions using the perfectly reasonable idea to first discretize the functions using a finite parameter vector $\theta\in\mathbb{R}^P$ and subsequently to view the objective function as a $\theta\to\mathbb{R}^P$ mapping, upon which standard finite dimensional methods can be applied.
Though useful, this is distinct from performing sensitivity analysis directly in the function space, which we now illustrate by example.
Let $u(\theta) := \theta_1 h_1 + (1-\theta_1) h_2 + \theta_2 h_3$, where $h_1,h_2,h_3$ is a dictionary of three functions; maybe they're $\sin$, $\cos$, and $\exp$.
Now say that our functional is: $f = \langle \cdot, h_1\rangle + \langle \cdot, h_2\rangle$.
According to a standard active subspace analysis with respect to $\theta$, there is only one important dimension: $\theta_1$. 
By contrast, the functional active subspace analysis with respect to $u$ as we will introduce in this article would give that $\textrm{span}(\{h_1,h_2\})$ is the two dimensional active subspace (assuming $h_1\neq ah_2 \,\,\forall a\in\mathbb{R}$).
Since the vector space dimension of these two active subspaces disagree, there cannot be a continuous mapping between them: they are truly different objects.
A similar inequivalence is present in the infinite dimensional optimization literature, where it is well known that discretizing then optimizing may lead to different results than optimizing then discretizing \cite[e.g.][]{liu2019non,gholami2019anode,onken2020discretize}.

There are also some nice practical advantages to working with the active subspace directly in the function space.
For one, we obtain eigenfunctions rather than eigenvalues. 
If the functions residing in the Hilbert space are defined on a low dimensional domain, they can be directly visualized in a plot; by contrast, interpreting an eigenvector beyond simply noting the magnitude of its largest elements can be difficult.
Additionally, staying in the function space means we can make comparisons across different mesh sizes or otherwise across discretization schemes, which is important in adaptive meshes \citep{nochetto2009theory}, or in any other circumstance where the number and meaning of the parameters in the discretization can vary.

\subsection{PDE-Constrained Optimization}

Though the active subspace methodology we develop is generally applicable, we will in this article primarily consider applications in applied math, and especially in modeling with differential equations.
Function-valued parameters can be of interest in such contexts in various ways. 
Sometimes, they are equated to a differential operator in a differential state equation of the form $\mathbb{D} v = l $, where $\mathbb{D}$ is some differential operator.
When $l$ does not depend on $v$, it is called a \textit{forcing function}.
A function may also specify an inhomogeneous Dirichlet or Neumann boundary condition.
These may be of interest to develop an optimal controller \citep{evans2024introduction} or to solve an inverse problem \citep{vogel2002computational}.
We next introduce three such problems from the applied engineering literature which serve as our example real-valued functionals, the sensitivity of which we want to compute. 
We give additional details in Appendix \ref{sec:app_app}.

\subsubsection{Distributed Poisson Control}

As an illustration, consider a Poisson Distributed Control problem (see Figure \ref{fig:illus}, left).
$\mathcal{X}=(0,1)\times(0,1)$ with homogeneous Dirichlet boundary conditions. 
Given a desired state $v_d(x,y) = \sin(4\pi x)\sin(\pi y)$ and regularization parameter $\alpha>0$, the suitability of a candidate control function $m$ is is given by:
\begin{equation}
  \label{eq:cost}
  J(m) = \frac12 \int_\dom (v-v_d)^2\,\dx
           + \frac{\alpha}{2} \int_\dom \dx,
\end{equation}
subject to the state equation:
\begin{equation}
  \label{eq:state}
  -\Delta v \;=\; m \quad \text{in } \mathcal{X}, 
  \qquad v \;=\; 0 \quad \text{on } \partial\mathcal{X} \,.
\end{equation}
Given a candidate setting $m$, the function $J$ measures its suitability for driving $v$ toward $v_d$ subject to $L^2$ regularization.
Modern computational platforms like FEniCS \citep{alnaes2015fenics,logg2012automated} allow for nearly exact calculation of functional gradients of $J$ at any value of $m$ under a variety of discretization schemes.

\subsubsection{Laminar Jet}

Following \citet{klein2003investigation,beskos2017geometric}, we study the problem of reconstructing the flow of a non-reacting laminar jet at the inlet boundary given 41 measurements on the outflow boundary.
This involves solving a Navier-Stokes equation over a 2D rectangular region to evaluate the resultant pressure and velocity fields.
In this problem the input function gives the 1D inflow profile of the jet, the forward model is run, and predictions at the boundaries are compared with ground truth data using mean squared error (see Figure \ref{fig:illus}, middle).

\subsubsection{Kirigami Electronics}

\citet{yang2024kirigami} propose the idea of \textit{Kirigami Electronics}, where Kirigami is, like Origami, a Japanese art form involving folded paper, but which also places cuts in the paper.
By shaping electronics in this manner, they unfurl into complex shapes based on their environment.
To develop their technology, \citet{yang2024kirigami} have developed a sophisticated computational model to determine what shape their intricate geometry will obtain in the steady-state balancing elastic energy and the load.
This problem involves specifying a 3D function over an irregular domain, and the cost function to be minimized is the energy of the configuration (see Figure \ref{fig:illus}, right).

%
%
%
%
%

\section{Active Subspaces in Hilbert Space}
\label{sec:defC}


Let $(\mathcal H,\langle\cdot,\cdot\rangle)$ be a real, separable Hilbert space and let
$f:\Omega\subset\mathcal H\to\mathbb R$ be a real-valued functional defined on an open set $\Omega$. 
We use extensively the notion of a derivative on a Hilbert space in this article by way of Fr\'echet and G\^ateaux differentiation.
Intuition is most readily gained when $\Hcal$ contains real-valued functions $u$ defined over some domain $\dom$.
Then, the functional derivative $Df(u)$ may be identified with a member of $\Hcal$, i.e., viewed as a real-valued function on $\dom$.
As with a classical derivative, evaluating it at some element $x\in\dom$ gives in some sense the infinitesimal change of $f(u)$ as $u(x)$ is varied. 
We now make this precise.
\begin{definition}
\label{def:nablaf}
Fix $u\in\Omega$. If $f$ is Fr\'echet differentiable at $u$, its derivative
$Df(u):\mathcal H\to\mathbb R$ is a bounded (continuous) linear functional. By the
Riesz representation theorem \citep{frechet1907, riesz1907espece, rudin1991}, there exists a \emph{unique} $\nabla \!f : \mathcal H\to\mathcal H$ such that for all $u \in \mathcal H$,
$\nabla \!f(u)\in\mathcal H$ and 
\begin{equation}
    \label{eq:Riesz}
    Df(u)[h] \;=\; \langle h,\; \nabla \!f(u)\rangle
\qquad\text{for all } h\in\mathcal H.
\end{equation}
We call $\nabla \!f(u)$ the (Fr\'echet) gradient of $f$ at $u$.
\end{definition}
Basic properties of $\nabla f$ are reviewed in Appendix \ref{sec:app_deriv}.

The integrand in the active subspace matrix is typically given as a vector times its transpose: $\nabla f(\x) \nabla f(\x)^\top$.
Taking the linear map perspective, we see that:
$
    \nabla f(\x) \nabla f(\x)^\top \mathbf{a} = 
    \langle \nabla f(\x), \mathbf{a}\rangle \nabla f(\x)\,.
$
This perspective immediately suggests an extension to the general Hilbert space setting.
We will denote this tensor product by $(a\otimes b)\,h := \langle h,b\rangle\,a$.

\begin{definition}
    Let $U$ be an $\mathcal H$-valued random variable with law $\rho$  on $\mathcal H$, and assume $\mathbb E \left[\|\nabla \!f(U)\|^2\right] <\infty$. Define the \emph{active subspace operator as:}
    \begin{equation}\label{eq:Cdef}
    \C \;:=\; \mathbb E\!\big[\nabla \!f(U)\otimes \nabla \!f(U)\big]
    \;\;:\;\; \mathcal H\to\mathcal H.
    \end{equation}
\end{definition}

In Appendix \ref{sec:app_well_defined}, we show that $\C$ is well defined as a Bochner expectation and is trace-class (and thus compact), self-adjoint, and positive semidefinite.
It therefore is governed by the Spectral Theorem.

\begin{proposition}[Spectral decomposition of $\C$]
\label{Theorem:spectral:C}
There exists orthonormal eigenfunctions $\{w_i\}_{i\ge1}\subset\mathcal H$ and a nonincreasing sequence 
$\lambda_1\ge\lambda_2\ge\cdots\ge0$ such that
\setlength{\columnsep}{-1cm}
\begin{multicols}{2}
    \begin{enumerate}
        \item $\C w_i=\lambda_i w_i$,
        \item $\mathrm{trace}(\C) = \underset{i\ge1}{\sum}\lambda_i.$
        \item $\C = \sum_{i\ge 1}\lambda_i\,w_i\otimes w_i$ 
        \item $\mathbb E\!\Big[\big(Df(U)[w_i]\big)^2\Big] = \lambda_i$
    \end{enumerate}
\end{multicols}
\end{proposition}
\begin{proof}
    Appendix \ref{sec:proof:spectral:C}
\end{proof}
\noindent Therefore, each $\lambda_i$ equals the average squared directional derivative of $f$ along the corresponding eigenfunction $w_i$, so larger $\lambda_i$ indicate directions with greater average variability of $f$. 

Thus we have provided an extension of the active subspace matrix.
Next, we give our definition of the active subspace itself.

\begin{definition}
Let the $n$-dimensional active subspace be given by $\mathcal A_n:=\mathrm{span}\{w_1,\dots,w_n\}$. 
\end{definition}

The value of $n$ is so chosen that the leading eigenvalues $\lambda_1,\dots,\lambda_n$ account for the dominant share of the spectrum. With the orthogonal decomposition $\mathcal H=\mathcal A_n\oplus\mathcal A_n^\perp$, projecting $u\in\mathcal H$ onto $\mathcal A_n$ retains the most informative coordinates and discards the (on average) insensitive ones. 

With the active subspace in hand, it is very useful to define a surrogate function which depends only on the active coordinates. 
To this end, for every $u\in\mathcal H$,
\begin{equation}\label{eq:surrogate}
f(u)\approx\tilde f\big(P_{\mathcal A_n}u\big)
=\tilde f\Big(\sum_{i=1}^{n}\langle u,w_i\rangle\,w_i\Big)\,,
\end{equation}
where $P_{\mathcal A_n}$ be the orthogonal projector onto $\mathcal A_n$.

We can extend the notion of a ridge function to infinite dimension.
In this case, $f(u) = \tilde{f}(\mathcal{M}u)$, where $\mathcal{M}:\Hcal\to\mathbb{R}^R$ is a linear operator.
In this case, $\C$ has rank at most $R$, its nonzero eigenfunctions span the range of $\mathcal{M}$, and $\mathcal A_n=\textrm{Range}(\mathcal{M})$ captures all gradient energy.

The following result is a Hilbert-space analogue of Proposition 2.3 in \citet{const:14}, which describes how ridge structure is reflected in the active subspace operator.

\begin{theorem}
\label{thm:HilbertProp23}
Fix $n\in\mathbb N$. Let $\Acal_n:=\mathrm{span}\{w_1,\dots,w_n\}$ and let
$P_{\Acal_n}$ denote the corresponding orthogonal projector.
Assume the trailing spectrum vanishes:
\begin{equation}\label{eq:trailingZero}
\lambda_{n+1}=\lambda_{n+2}=\cdots=0.
\end{equation}
Then the following hold.
\begin{enumerate}
\item[(i)] (\emph{Gradient lies in the active subspace}) For all $u\in\Hcal$,
$P_{\Acal_n^\perp}\nabla \!f(u)=0$ (equivalently, $\nabla \!f(u)\in\Acal_n$ everywhere).

\item[(ii)] (\emph{Level-set invariance})
If $u_1,u_2\in\Hcal$ satisfy $P_{\Acal_n}u_1=P_{\Acal_n}u_2$, then $f(u_1)=f(u_2)$.

\item[(iii)] (\emph{Gradient equality})
Under the same hypothesis, $\nabla \!f(u_1)=\nabla \!f(u_2)$.

\item[(iv)] (\emph{Factorization through $P_{\Acal_n}$})
There exists a $C^1$ map $\tilde f:P_{\Acal_n}(\Hcal)\to\mathbb R$ such that
$
f(u)=\tilde f\!\big(P_{\Acal_n}u\big)\
$
for all 
$
u\in\Hcal, \ \text{and}
\
\nabla\tilde f\!\big(P_{\Acal_n}u\big)=\nabla \!f(u)\in\Acal_n.
$
\end{enumerate}
\end{theorem}
\begin{proof}
    Appendix \ref{proof:HilbertProp23}
\end{proof}
\medskip

Next, we broaden our discussion to functions which do not enjoy ridge structure. 
We show how we can nevertheless construct low-dimensional approximations which enjoy low MSE under certain assumptions.
For a random element $U$ with law $\rho$
on $\Omega\subset\Hcal$, define
$
Y:=P_{\Acal_n}U\in\Acal_n$, 
$Z:=P_{\Acal_n^\perp}U\in\Acal_n^\perp$,
so that $U=Y+Z$.
Denote by $\rho_Y$ the law of $Y$.
We begin by defining an Active-surrogate extending \eqref{eq:surrogate} by conditional expectation:
\begin{equation}
\label{def:FY}
F(y)\;:=\;\E\big[f(U)\mid Y=y\big],\qquad y\in\Acal_n,
\end{equation}
Then $F$ is the $L^2(\rho)$-best predictor of $f$ among functions of $Y$.
We now demonstrate a bound on the error introduced by ignoring the remaining directions, extending Theorem 3.1 of \citet{const:14}.

\begin{theorem}[Hilbert-space MSE bound for the active-subspace surrogate]\label{thm:H-AS-MSE}
Under Assumption \emph{\ref{ass:A0}} and $\E\|\nabla \!f(U)\|^2<\infty$, the mean squared error of the active-surrogate F defined in \eqref{def:FY} satisfies
\begin{equation}\label{eq:H-AS-MSE}
\E\big( f(U)-F(Y)\big)^2\;\le\; C_\star\;\sum_{i>n}\lambda_i
\;=\; C_\star\;\E\big\|\Proj_{\Acal_n^\perp}\nabla \!f(U)\big\|^2.
\end{equation}
\end{theorem}
\begin{proof}
    Appendix \ref{sec:proveH-AS-MSE}.
\end{proof}

\citet[Section 3.1.2;][]{wycoff2022sensitivity} showed how to establish a transformation which equalizes the active subspace importance of input dimensions with nonzero impact.
In particular, they showed that if  $\C^{1/2}$ is positive matrix square root of $\C$, then the mapping between the ``warped" input $V$ transformed by $\C^{\frac1{2}}$ and the output $f(U)$ is equally sensitive to the input directions in the range of $\C$.
Our next result extends this to the Hilbert space setting.

\begin{theorem} [Active directions equalized by $\C^{1/2}$] 
\label{thm:rotate}
Assume $ \mathbb{E}\,\|\nabla \!f(U)\|^2<\infty$ and Assumption \ref{ass:inactive-invariance} holds.
Let 
$ \mathcal{B}\ :=\ \overline{\operatorname{ran}(\C^{1/2})}\ =\ (\ker \C)^\perp$ and let $\ P_{\mathcal{B}}$ denote the orthogonal projector onto $\mathcal{B}$. 
Then:
\begin{align}
\mathbb E_{\rho_V}\big[\nabla_v \tilde f(V)\otimes \nabla_v \tilde f(V)\big] = P_{\mathcal{B}}\,,
\end{align}
where $V=\C^{\frac1{2}} U$ with distribution $\rho_V$ implied by $\rho$.
\end{theorem}
\begin{proof}
    Appendix \ref{sec:proof:rotate}
\end{proof}

\section{Monte Carlo Estimation of $\C$}

We have thus shown that with $\C$ in hand, we can perform dimension reduction. 
However, we will typically not have an analytic expression for $\C$. 
Following the typical workflow in the finite dimensional setting \citep{constantine2014computing}, we will in this section develop and analyze a Monte-Carlo estimator $\Ch$ for $\C$.


\begin{definition}
Let $U_1,\dots,U_B\overset{\text{i.i.d.}}{\sim}\rho$. Set $g_b:=\nabla \!f(U_b)\in\Hcal$, $b=1,\dots,B$, and let $h\in\Hcal$. Define the empirical covariance operator
\[
\widehat \C_B \;:=\; \frac{1}{B}\sum_{b=1}^B g_b\otimes g_b\in\mathcal L(\Hcal);
\,\,
\widehat \C_B h=\frac{1}{B}\sum_{b=1}^B \langle h,g_b\rangle\,g_b.
\]
\end{definition}

Though this is an operator on $\Hcal$, it is of finite rank, and its action is computable via $B$ many $\Hcal$ inner products. 
This is precisely what is straightforwardly approximated in modern numerical environments. 

We next show how to compute its eigendecomposition.
To this end, we define the semi-infinite matrix 
$G_B:\mathbb{R}^B \to \mathcal{H}$ by $G_B e_b=g_b/\sqrt{B}$ so that $\widehat \C_B=G_B G_B^{*}$, and let $\Gamma_B:=G_B^{*}G_B\in\mathbb{R}^{B\times B}$.
Perform the eigendecomposition $\Gamma_B = \sum_{i=1}^B \sigma_i v_i v_i^\top$.

\begin{lemma}[Eigenanalysis for $\Ch$]
\label{lem:gram-spectral}
For all $i\in\{1, \ldots, B\}$, define $\hat w_i := \frac{G_b v_i}{\sqrt{\sigma_i}}$.
Then the $\hat w_i$ are $\Hcal$-orthogonal, $\Ch \hat w_i = \sigma_iw_i$, and:
$
\widehat \C_B=\sum_{i=1}^r \sigma_i\,\widehat w_i\otimes \widehat w_i.
$
\end{lemma}
\begin{proof}
    Appendix 
\ref{sec:proof_gram-spectral}.
\end{proof}

We see that an eigenanalysis of the matrix $\Gamma_B$ unlocks an eigenanalysis for the operator $\Ch$. 
Their nonzero eigenvalues are shared, and the eigenfunctions of $\Ch$ are linear combinations of the sampled gradients $g_b$ with coefficients given by $[\frac{v_{i,1}}{\sigma_i}, \ldots ,\frac{v_{i,B}}{\sigma_i}]$. See Appendix \ref{sec:proof_gram-spectral} for a more detailed statement of Lemma \ref{lem:gram-spectral}.

Our next result establishes convergence of this Monte Carlo estimator.

\begin{theorem}[Consistency]
\label{thm:cons}
    Assume $\E\|\nabla \!f(U)\|^2<\infty$. 
    Then $\|\widehat \C_B-\C\|_{\mathrm{op}}\to0$ almost surely.
\end{theorem}

\begin{proof}
    Appendix \ref{sec:proof_cons}.
\end{proof}

Since the active subspace operator itself is frequently used as a means for obtaining the active subspace, we next investigate the convergence behavior of an eigenanalysis on $\Ch$ to that on $\C$.
Recall that $(\lambda_i,w_i)_{i\geq 1}$ are the eigenpairs of $\C$ and $(\sigma_i, \hat w_i)_{i\in{1,\ldots,B}}$ are those of $\Ch.$


\begin{corollary}\label{cor:eig-consistency}
Assume $\E\|\nabla \!f(U)\|^2<\infty$. Then:

\begin{enumerate}[label=(\roman*),leftmargin=2em]
\item \emph{Eigenvalues.} For every $i\ge1$,
\[
|\sigma_i-\lambda_i|\ \xrightarrow[B\to\infty]{\mathrm{a.s.}}\ 0.
\]

\item \emph{Eigenfunctions.} 
Fix $i\ge1$ and suppose $\lambda_i$ is simple and isolated with spectral gap
\[
\gamma_i\;:=\;\min\{\lambda_{i-1}-\lambda_i,\ \lambda_i-\lambda_{i+1}\}\;>\;0
\quad(\lambda_0:=+\infty).
\]
Then $\|s_i\hat w_i-w_i\|\xrightarrow[B\to\infty]{\mathrm{a.s.}}0$ for appropriately chosen signs $s_i\in\{-1,1\}$.
\end{enumerate}
\end{corollary}
\begin{proof}
    Appendix \ref{app:proof_eig-consistency}.
\end{proof}

\medskip

We next determine the convergence rate under an additional moment assumption.

\begin{theorem}[$\sqrt{B}$-rate of convergence]\label{thm:MSE}
Assume $\E\|\nabla \!f(U)\|^4<\infty$. Then
\begin{align}
\E\big\|\widehat \C_B - \C\big\|_{op}=O(B^{-1/2}).
\end{align}
\end{theorem}
\begin{proof}
    Appendix \ref{sec:proof_MSE}, which also presents Chebyshev bounds.
\end{proof}

While the $\sqrt{B}$ convergence rate is typical, it also means we get diminishing marginal returns. 
This makes a uncertainty quantification about the $\C$ estimate itself practically useful.
Our final result concerns the asymptotic distribution of $\Ch$.

\begin{theorem}[Hilbert--space CLT in $\HS(\Hcal)$]\label{thm:CLT}
Under Assumption \ref{ass:A1} of the Appendix, in the separable Hilbert space $\HS(\Hcal)$,
\[
\sqrt{B}\,\big(\widehat \C_B - \C\big)\ \Rightarrow\ Z_\infty,
\]
where $Z_\infty$ is a mean-zero Gaussian element in $\HS(\Hcal)$ with covariance operator
given in the Appendix.
\end{theorem}

\begin{proof}
    Appendix \ref{sec:proof_CLT}
\end{proof}

\begin{figure*}
    \centering
    \newcommand{\scalefiguretwo}{0.28}

    \begin{tabular}{|c|c|c|c|}
    \hline
    &
    \texttt{poisson} &
    \texttt{laminar} &
    \texttt{kiri} \\
    \hline
    $\lambda$ & 
    \includegraphics[width=\scalefiguretwo\linewidth,valign=m]{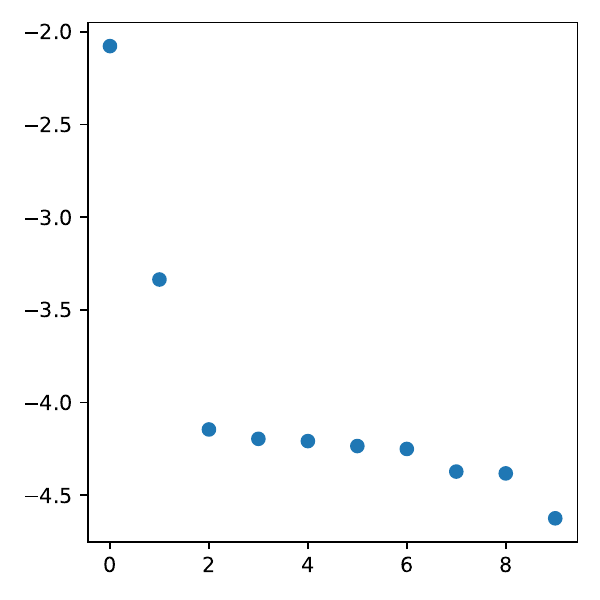}&
    \includegraphics[width=\scalefiguretwo\linewidth,valign=m]{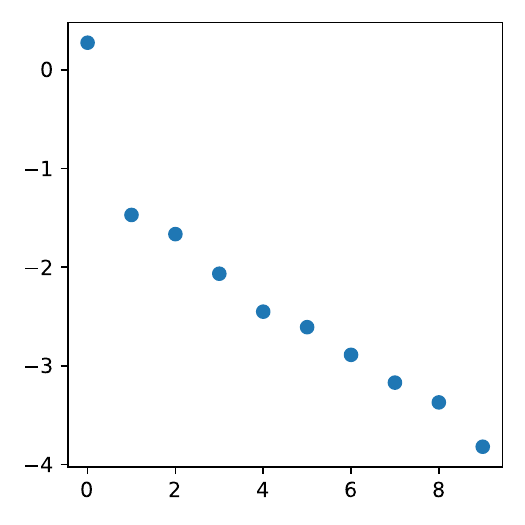}&
    \includegraphics[width=\scalefiguretwo\linewidth,valign=m]{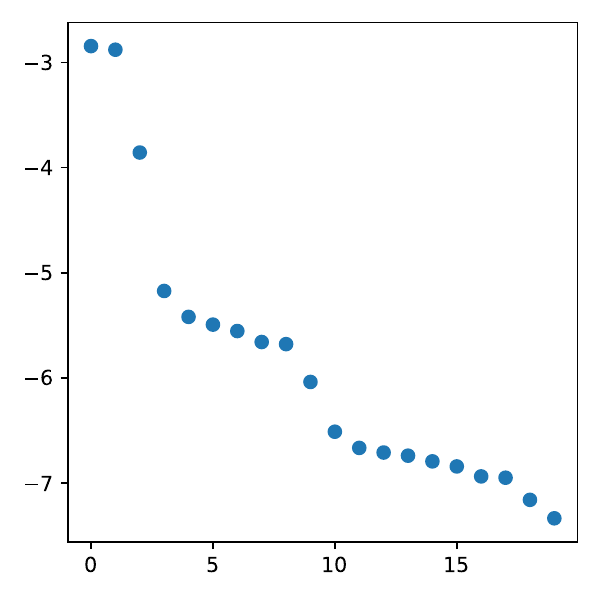} 
    \\
    \hline
    $w_1$ & 
   \includegraphics[width=\scalefiguretwo\linewidth,valign=m]{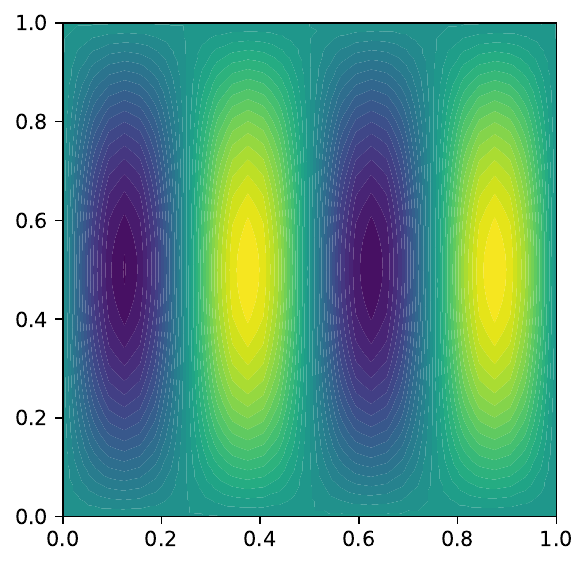}&
   \includegraphics[width=\scalefiguretwo\linewidth,valign=m]{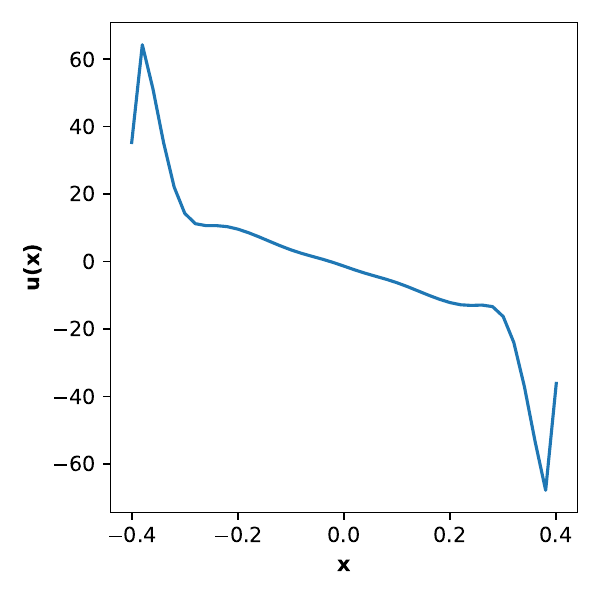}&
   \includegraphics[width=\scalefiguretwo\linewidth,valign=m,trim={6em 6em 6em 6em},clip,]{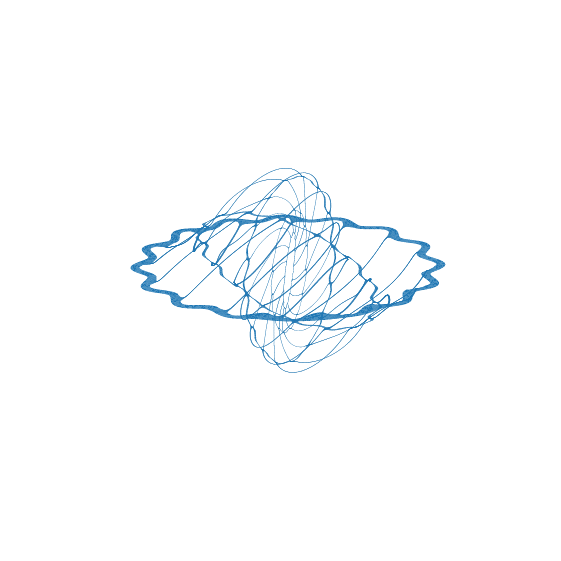}
   \\
    \hline
    $w_2$ & 
   \includegraphics[width=\scalefiguretwo\linewidth,valign=m]{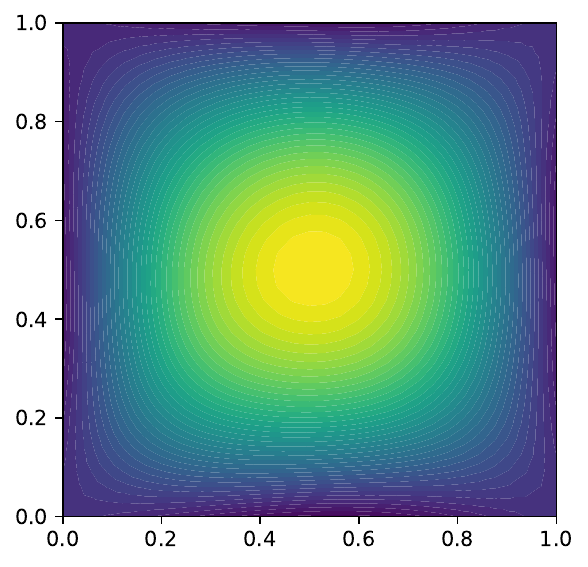}&
   \includegraphics[width=\scalefiguretwo\linewidth,valign=m]{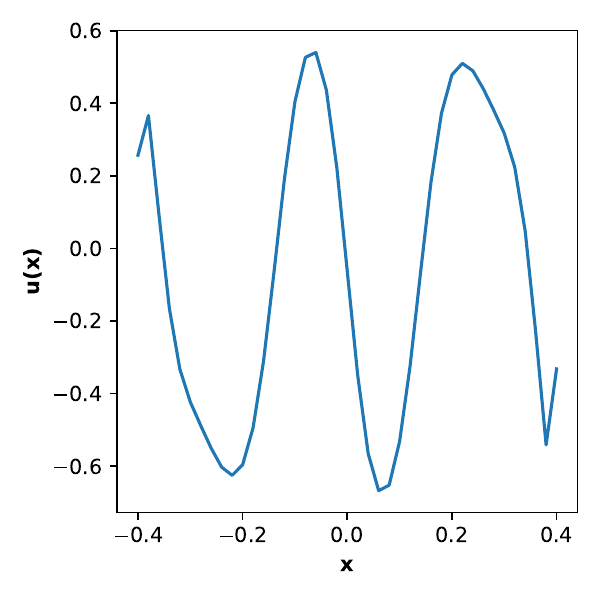}&
   \includegraphics[width=\scalefiguretwo\linewidth,valign=m,trim={6em 6em 6em 6em},clip,]{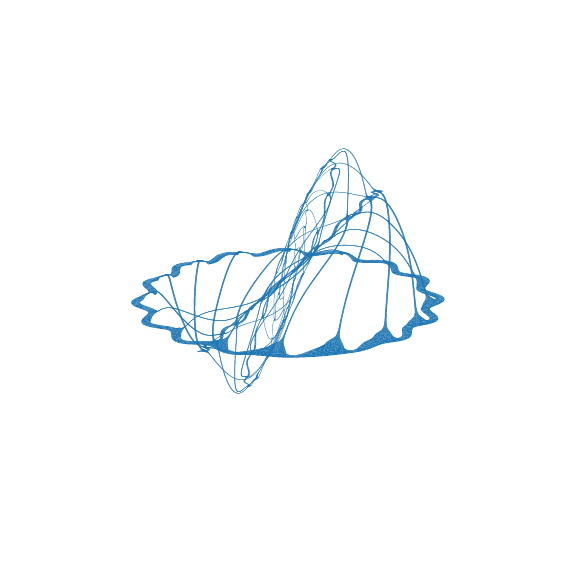}
   \\
    \hline
    $P_{W_{12}}$ &
   \includegraphics[width=\scalefiguretwo\linewidth,valign=m]{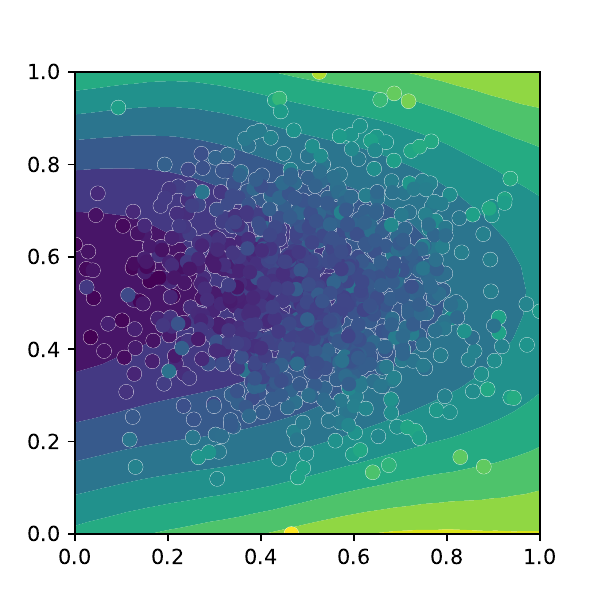} &
   \includegraphics[width=\scalefiguretwo\linewidth,valign=m]{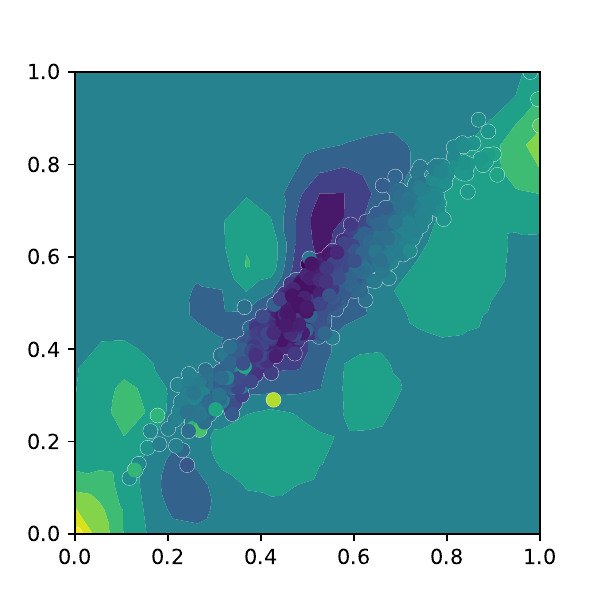} &
   \includegraphics[width=\scalefiguretwo\linewidth,valign=m]{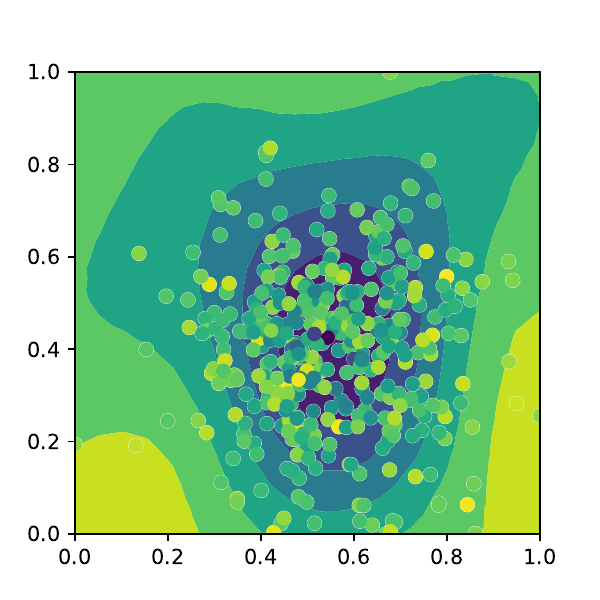}\\
   \hline
    \end{tabular}

    \caption{
    \textbf{Visualization with Active Subspaces.}
    Each row corresponds to a function.
    \textit{Top row:} Monte-Carlo estimates of first 10 eigenvalues of active subspace operator.
    \textit{Middle rows:} Estimate for first and second eigenfunctions.
    \textit{Bottom row:} Projection of samples along first two eigenfunctions (colored points); surface gives a Gaussian process conditional mean estimate of the $L^2$-optimal surrogate fit using \texttt{hetGPy} \citep{o2025hetgpy}.
    }
    \label{fig:viz}
\end{figure*}

\section{Applications}

We now give deploy our methodology on our motivating test functions and discuss the results. 

\subsection{Exploratory Analysis and Visualization}

The first question an analyst will have is whether an active subspace exists and if so what its dimension is.
\citet{constantine2015active} suggests to look for a significant gap in the log-eigenvalues, while traditionally for PCA, cumulative variance retained being above some threshold is used to determine the reduced subspace dimension.
As in finite dimension, this may in either case be achieved by visually inspecting the decay of the spectrum of the operator.
We show the estimated spectra in the top row of Figure \ref{fig:viz}.

A key application of the active subspace method is visualization of high dimensional functions.
Just as in finite dimension, we can project some functions $u_1, \ldots, u_N$ onto the two leading eigenfunctions
$w_1,w_2$.
Coloration can be used to indicate the value $f(u_i)$ of each point. 
However, if there is significant variation in $f$ beyond the first two eigenfunctions, this may lead to a difficult to interpret plot.
Therefore, we can enrich the plot by fitting a statistical model to the reduced dataset using $\langle u_i, w_1\rangle$, $\langle u_i, w_2\rangle$ as predictor variables being used to predict $f(u_i)$, and treat the remaining variation in the function as noise. 
A plot of the model's predictions can be used as an estimate of the optimal $L^2$ active subspace surrogate. 
We show these surfaces in the bottom row of Figure \ref{fig:viz}.
For the \texttt{poisson} and \texttt{laminar} functions, we see that most of the variation is captured by the low dimensional plot.
Two dimensions is not sufficient to capture all variation for the \texttt{kiri} function, but we nevertheless observe some low-dimensional, bowl-shaped structure.

\begin{figure}
    \centering
    \includegraphics[width=0.3\linewidth,trim={1.15em 1em 1.0em 1em},clip]{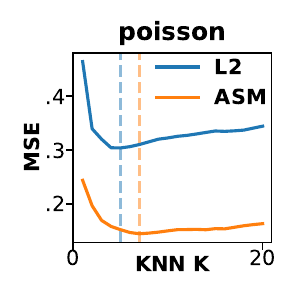}
    \includegraphics[width=0.3\linewidth,trim={1.15em 1em 1.0em 1em},clip]{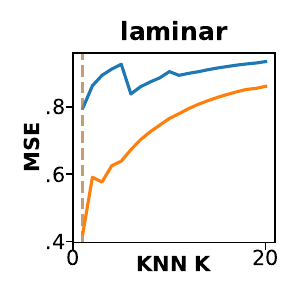}
    \includegraphics[width=0.3\linewidth,trim={1.15em 1em 1.0em 1em},clip]{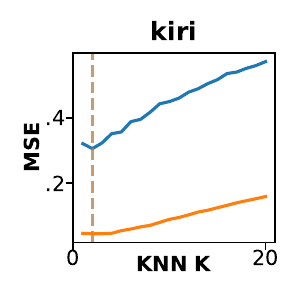}
    \caption{\textbf{Active Subspace for functional KNN Regression}.
    Cross Validation MSE comparison for KNN with standard $L^2$ distance versus Euclidean distance along active subspace projection.
    }
    \label{fig:knn}
\end{figure}

\subsection{Surrogate Modeling}
\label{sec:app_prewarp}

Active subspaces can also serve to increase the predictive accuracy of local surrogate models.
Following \citet{fukumizu2014gradient}, we test our dimension reduction method by evaluating its ability to improve the performance of a K Nearest Neighbors (KNN) regression.
KNN regression is most straightforwardly applied to our functional context by predicting the output of $f$ at some unobserved point $u^* \in \mathcal H$ by finding the $K$ training points with the least $L^2$ distance to $u^*$.
This will serve as our baseline method.

To deploy our active subspace operator in this context, we first project functions at which predictions are desired onto the active subspace $\Acal$, and then compute standard Euclidean distance. 
In particular:
\begin{align}
    \Vert u_1 - u_2 \Vert_{\mathcal{A}}
    :=
    \Vert P_{\mathcal{A}} (u_1 - u_2) \Vert_{L^2} \,.
\end{align}
This is computed by computing:
$
    \Omega = \G^\top\V \in \mathbb{R}^{B\times B} \,
$
and then computing the distance between coordinates:
\begin{align}
    \Vert u_1 - u_2 \Vert_{\Ch}
    =
    \sqrt{\sum_{b=1}^B(\omega_{1,b}-\omega_{2,b})^2} \,.
\end{align}

Using leave-one-out Cross Validation, we compare the predictive error of KNN using the standard $L^2$ distance versus the distance along the active subspace for various $K\in\{1, \ldots, 20\}$ on our test functions.
Figure \ref{fig:knn} gives the results.
We see that on all three functions, the active subspace distance is better from a predictive perspective at the optimal value of $K$. 
For the \texttt{poisson} and \texttt{kiri} problems, the worst performing value of $K$ for the active subspace method is superior to the best performing value of $K$ for the $L^2$ method.

\subsection{Bayesian Optimization}

Bayesian Optimization \citep[BO;][]{jones1998efficient, garnett2023bayesian} has been classically viewed as a method appropriate only for mild dimensional input spaces \citep{binois2022survey}, and thus it may be unsurprising that its application to function spaces has been somewhat limited in the academic literature.
However, building on \citet{wang2016bayesian} working in the finite dimensional case,
\citet{vellanki2019bayesian,shilton2020sequential} propose to methods to conduct Bayesian optimization on spaces of functions by searching within random finite dimensional subspaces.
These subspaces are generated as the span of a randomly drawn set of functions $m_1, \ldots, m_R \sim \rho$ from some measure on $\Hcal$.
This choice of $\rho$, together with the target function, implicitly defines an active subspace operator.
We will propose to use as a basis for the search space  not $\{m_1, \ldots, m_R\}$ as in previous approaches, but rather the leading eigenbasis of $\Ch = \sum_{r=1}^R \nabla f(m_r)\otimes \nabla f(m_r)$. 

We define a new function $g:[-1,1]^R\to\mathbb{R}$ as:
\begin{equation}
    g(\mathbf{c}) = \sum_{i=1} \frac{c_i}{\ell} q_i \,,
\end{equation}
where $q_i$ are either the randomly generated functions from $\rho$ (the \texttt{Rand} method) or the leading $R$ eigenfunctions of $\Ch$ (the \texttt{ASM} method).
This gives us a function defined on a finite dimensional cube which makes it compatible with standard BO machinery. 
We use in particular the BO framework Ax \citep{olson2025ax}.
We sampled $\Ni =10$ random initial functions $m_1, \ldots, m_{\Ni}$ from $\rho$ and then projected them onto $\{q_1, \ldots, q_R\}$.
The coefficients can be computed by $(Q^*Q)^{-1}Q^* M $, where $Q^*Q$ is the gram matrix of $\{q_1, \ldots, q_R\}$ and $Q^*M$ is the matrix of inner products between the two sets of functions.
The scalar $\ell$ is 1.5 times the maximum absolute value of any coefficient.

We then evaluated the objective function $g$ at the $\Ni$ projected initial functions. 
Although for each repetition we use the same initial points, the initial design for the finite dimensional BO subroutine are different because the projection of this initial set onto the respective bases differs.
After the $\Ni$ initial random functions, we used an additional 40 sequential evaluations chosen by the BO subroutine using expected improvement.
We conducted 100 repetitions with different random seeds for each function and method.

We present the $10^{th}, 50^{th}$ and $90^{th}$ percentiles of the optimization progress in Figure \ref{fig:bo}.
On the \texttt{poisson} problem, the \texttt{Rand} method is barely able to make progress, while the \texttt{ASM} method quickly finds the global optimum (we show only the first 20 iterations to better illustrate the difference).
On the \texttt{laminar} problem, the initial set of functions generated by the random projection actually generates better starting output values, but in the sequential design phase the \texttt{ASM} method rapidly surpasses the \texttt{Rand}.
On the \texttt{kiri} problem, both methods quickly find that the zero function gives a decent solution (as the midpoint of the space, Ax specifically checks this point at initialization), but the \texttt{ASM} method much more quickly discovers improved solutions.
Overall, the active subspace approach quantitively surpasses the random subspace method.

\begin{figure}
    \centering
    \includegraphics[width=0.3\linewidth,trim={1.15em 1em 1.0em 1em},clip]{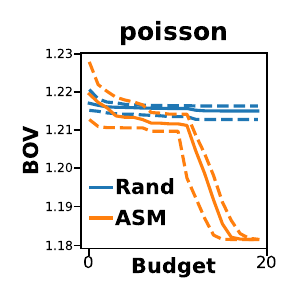}
    \includegraphics[width=0.3\linewidth,trim={1.15em 1em 1.0em 1em},clip]{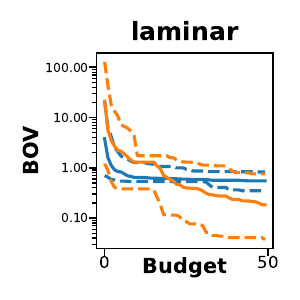}
    \includegraphics[width=0.3\linewidth,trim={1.15em 1em 1.0em 1em},clip]{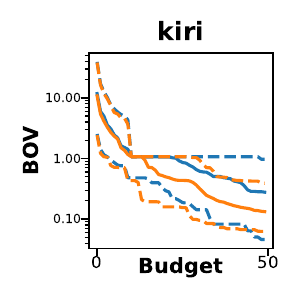}
    \caption{\textbf{Functional Optimization in Active Subspace.}
    Compares searching a random vs active subspace; y-axis gives Best Observed Value.
    Solid line gives median and dotted $10^{th}$ and $90^{th}$ percentile.
    }
    \label{fig:bo}
\end{figure}

\section{Discussion}

\paragraph{Summary and Conclusions.}
In this article, we introduced an extension of the active subspace matrix to infinite dimensions, established its theoretical properties, and developed a computable approximation.
We found that the active subspace method can be a powerful tool for developing finite dimensional approximations of functionals on a Hilbert space and dominated existing approaches on our complex test functions.

\paragraph{Future Work.}
In the decade since the active subspace method has swept computational engineering, authors have proposed a number of extensions and improvements.
For instance, several authors have proposed active subspaces for vector-valued functions
\citep{zahm2020gradient,tripathy2019deep,edeling2023deep,musayeva2024shared,rumsey2025co}
which could be ported to our infinite dimensional input setting.
\citet{lam2020multifidelity} consider active subspaces for multifidelity models.
Several authors have proposed strategies for estimating active subspaces without gradients using statistical models \citep[e.g.][]{wycoff2021sequential,rumsey2024discovering}, and extending these to the infinite dimensional setting would expand the applicability of this method.
Finally, \citet{lee2019modified} propose a modification of the active subspace matrix which emphasizes the mean gradient, and it would be interesting to study a functional analog.

\bibliography{main}

\end{document}